\documentclass[twoside]{article}

\usepackage[accepted]{aistats2021}
\usepackage[round]{natbib}

\usepackage{amssymb,amsmath,xcolor}

\usepackage{amssymb,amsmath,natbib,todonotes}
\usepackage{algorithmic,algorithm}
\usepackage{hyperref}

\newcommand{\BEAS}{\begin{eqnarray*}}
\newcommand{\EEAS}{\end{eqnarray*}}
\newcommand{\BEA}{\begin{eqnarray}}
\newcommand{\EEA}{\end{eqnarray}}
\newcommand{\BEQ}{\begin{equation}}
\newcommand{\EEQ}{\end{equation}}
\newcommand{\BIT}{\begin{itemize}}
\newcommand{\EIT}{\end{itemize}}
\newcommand{\BNUM}{\begin{enumerate}}
\newcommand{\ENUM}{\end{enumerate}}

\newcommand{\BA}{\begin{array}}
\newcommand{\EA}{\end{array}}

\newcommand{\eg}{e.g., }
\newcommand{\ie}{i.e.}

\newcommand{\reals}{{\mathbb R}}

\newcommand{\Tr}{\mathop{\bf Tr}}

\newcommand{\idm}{\mathbf{I}}
\newcommand{\dom}{\mathop{\bf dom}}

\newcommand{\la}{\langle}
\newcommand{\ra}{\rangle}

\newtheorem{theorem}{Theorem}[section]
\newtheorem{proposition}[theorem]{Proposition}
\newtheorem{definition}[theorem]{Definition}
\newtheorem{lemma}[theorem]{Lemma}

\newenvironment{proof}{\textbf{Proof.}}{\hfill$\Box$\bigskip}

\newcommand{\PA}{\mathrm{PA}} 
\newcommand{\normF}[1]{\|#1\|} 
\newcommand{\normone}[1]{\|#1\|_1} 
\DeclareMathOperator*{\minim}{\mathrm{min.}} 
\DeclareMathOperator*{\argmin}{\mathop{argmin}}

\newcommand{\Expect}{\mathbb{E}}

\newcommand{\SHD}{\mathrm{SHD}}
\newcommand{\FDR}{\mathrm{FDR}}

\begin{document}

%
\runningtitle{A Bregman Method for Structure Learning on Sparse Directed Acyclic Graphs}

%

\twocolumn[

\aistatstitle{A Bregman Method for Structure Learning on\\ Sparse Directed Acyclic Graphs}

\aistatsauthor{Manon Romain \And Alexandre d'Aspremont}

\aistatsaddress{CNRS, ENS Paris \& Inria \And CNRS \& ENS Paris} ]

\begin{abstract}
We develop a Bregman proximal gradient method for structure learning on linear structural causal models. While the problem is non-convex, has high curvature and is in fact NP-hard, Bregman gradient methods allow us to neutralize at least part of the impact of curvature by measuring smoothness against a highly nonlinear kernel. This allows the method to make longer steps and significantly improves convergence. Each iteration requires solving a Bregman proximal step which is convex and efficiently solvable for our particular choice of kernel. We test our method on various synthetic and real data sets.
\end{abstract}

\section{INTRODUCTION}

Estimating directed acyclic graphs (DAGs) from observational data is a problem of rising importance in machine learning, with applications in biology \citep{Sach2005}, genomics \citep{Hu2018}, economics \citep{Imbe2019}, time-series analysis \citep{mali2018} and causal inference \citep{Pear2009,Pete2017}. More precisely, given $n$ random variables $X_1\ldots, X_n$, we seek to learn the structure of a Structural Causal Model (SCM), written
\BEQ
  X_j = f_j(X_{\PA(j)}, \epsilon_j),
\nonumber \EEQ
where $\epsilon_j$ is random noise and $\PA(j)$ are the parents of $X_j$, the subset of $\{X_1, \dots, X_n\}$ on which $X_j$ depends. This set of dependencies naturally forms a graph where $\{X_j\}_{j=1\ldots,n}$ are nodes, and edges are directed from nodes in $\PA(j)$ to $X_j$. In this context, we assume the graph $G$ is acyclic, see \eg  \citep[Def.\,6.2]{Pete2017}. 

Equivalently, we could describe our problem in the framework of Bayesian networks, where we search for a compact factorization of the joint distribution. Using the chain rule, we write,
\BEQ
  P(X_1\ldots, X_n) = \prod_{j=1}^n P(X_j\,|\,X_{\PA(j)}),
\nonumber \EEQ
where $\PA(j)$ are the parents of $X_j$: a subset of $\{X_1, \dots, X_{j-1}\}$ on which $X_j$ depends. Unlike with SCMs, the underlying graph $G$ is here naturally acyclic. Structure learning then consists in learning an adequate permutation of the order of variables and the sets $\{\PA(j)\}_{j=1\ldots,n}$. The combinatorial nature of this task makes it NP-hard \citep{Chic2004}. In cases where the graph is identifiable, structure learning of the Bayesian network is equivalent to learning actual causal relationships from observational data, we refer the reader to \citep{Pear2009,Pete2017} for a more complete discussion. 

Here, we focus on the well-studied (and simpler) case of {\em linear structural causal models} of the form,
\BEQ\label{eq:linear_sem}
  X_j =\sum_{k=1}^n\beta_{jk}X_k + \epsilon_j, \quad \mbox{for $j=1\ldots,n$,}
\EEQ
where \emph{parents} of $X_j$ are ${\PA(j) =\{X_k: \beta_{jk}\neq 0\}}$, $\beta_{jj} = 0$ and $(\epsilon_j)_{j=1\ldots, n}$ are mutually independent, centered \ie $\Expect(\epsilon_j)=0$ and independent from variables in $\PA(j)$. We do not assume that the noise is Gaussian, and aim for a model that works across noise distributions. Structure learning here means searching for the weighted adjacency matrix $W^* = (\beta_{jk})$ of the directed acyclic graph $G$, hence constraining our optimization variable to be in the space of DAGs. 

The structure learning problems we tackle here are non-convex and penalizing to impose acyclicity also gives them high curvature. While nonconvexity is somewhat unavoidable, we focus on taming curvature, which severely limits the performance of classical gradient methods by forcing them to take short steps. Bregman gradient algorithms along the lines of \citep{Birn11,Baus16,Lu18} extend projected gradient methods using a Bregman proximal step instead of a classical projection with respect to a norm, which pushes much of the curvature in the proximal step. Implicitly, using a highly nonlinear kernel allows the method to form a better local model of the function, so the method behaves very much like gradient descent in well conditioned settings, taking longer steps. In practice then, provided certain relative smoothness conditions \citep{Baus16,Bolt18} are satisfied and the Bregman projection step can be solved efficiently (which is the critical part), Bregman gradient methods allow us to neutralize part of the curvature by measuring smoothness against a nonlinear kernel and improve convergence.

\subsection{Related Work}
Structure Learning methods historically divide into constraint-based methods that test for conditional independence relations and score-based methods that optimize a variety of heuristics. 

\paragraph{Constraint-based} In this approach, we test multiple conditional dependencies. Finding two conditionally independent variables ${X \perp\!\!\!\perp Y\,|\, Z}$ means all paths from $X$ to $Y$ are of form ${X\to \dots\to Z \to \dots \to Y}$ or ${X\leftarrow \dots\leftarrow Z \to \dots \to Y}$ \citep{koll2009}. Under restrictive hypotheses (such as faithfulness), a graph can be constructed from such relationships. A popular example of this approach is the PC algorithm \citep{kali2007}. 

\paragraph{Score-based} In a typical score-based method, a \emph{discrete} scoring function is optimized over the space of DAGs. Scoring functions can be penalized likelihood such as BIC or coming from a Bayesian approach like BDeu \citep{heck1995}. Greedy-hill climbing is popular to optimize such scores. GES \citep{Chic2004} reduces the search space to Markov equivalences classes. \cite{Van-13} use the similar $\ell_0$ penalized maximum likelihood estimator proven to be consistent in high dimensions, under favorable assumptions. Those approaches usually require a form of faithfulness which is a restrictive assumption as shown by \cite{Uhle2013}. Finally, some \emph{hybrid} methods alternate between score optimization and constraint-based updates \citep{Rask13}. However the combinatorial nature of the DAG space makes structure learning computationally challenging. In this work, we will use a particular penalty to impose DAG structure to the graph.

\paragraph{NOTEARS} \cite{Zhen2018} proposed a novel smooth characterisation of the acyclicity of an adjacency matrix ${W \in \reals^{n\times n}}$, 
\BEQ
\Tr \exp(W\circ W) - n = 0,
\nonumber \EEQ
where $\exp(\cdot)$ is the matrix exponential. This allows solving a continuous optimization problem over the whole space $\reals^{n\times n}$ subject to the acyclicity constraint, instead of a discrete DAG problem, hence use off-the-shelf solvers for smooth non-convex optimization. Unfortunately, while the function $\Tr \exp$ is convex on the space of symmetric positive definite matrices $\mathbb{S}^n$, this property does not extend to $\reals^{n\times n}$. New methods have used this constraint in association with several popular deep learning methods \citep{Yu19,ng2019masked,ng2019ae,lachapelle2020gradientbased,Ng2020}. With the exception of the original NOTEARS, all listed methods were designed for GPUs and require substantial computational power.

\paragraph{Contributions} Based on penalty terms derived in, \eg \cite{Zhen2018}, we use a non-convex Bregman composite optimization framework, to produce a more efficient algorithm for structure learning in linear SCMs. Our choice of Bregman kernel means that the method behaves as a better conditioned gradient method, and that each iteration of the Bregman proximal gradient method requires solving a convex optimization subproblem. We demonstrate the empirical effectiveness of a soft DAG constraint and competitive results even in high-dimensional settings ($m\ll n$) for a fraction of the time taken by NOTEARS.


The paper is organized as follows. We first introduce the problem of structure learning with DAG penalties and after explaining the general theory of the Bregman proximal gradient methods, we show how our task fits into this framework. We then demonstrate the effectiveness and numerical performance of our approach on several synthetic data sets and show intuitive results on real data sets.

\paragraph{Notations} We use $\normF{\cdot}$ to denote the Frobenius norm on matrices and unusually $\normone{\cdot}$ is the sum of absolute values of all the matrix coefficients \ie $\normone{W} = \sum_{i,j} |W_{ij}|$ (this make notations more compact).

\section{STRUCTURE LEARNING PROBLEM}

The linear structure learning problem is formulated as follows. Let $X$ be a $m\times n$ matrix of $m$ i.i.d. observations from a linear structural causal model \eqref{eq:linear_sem}. We can write \eqref{eq:linear_sem} in matrix form
\BEQ\BA{lr}
X=XW^*+E
,\EA
\EEQ
where $W^*=(\beta_{jk})\in \reals^{n\times n}$ is the weighted adjacency matrix of the underlying directed acyclic graph $G$ and $E=(\epsilon_j^{(i)})\in\reals^{m\times n}$ are $m \times n$ independent noise samples. We use the least-squares loss
\BEQ\BA{rl}\ell(W; X) &= \frac{1}{m}\normF{X(\idm-W)}^2,
\EA\EEQ
however, everything that follows applies to any differentiable convex loss $\ell: \reals^{n\times n}\to \reals$. With finite samples and in high-dimensions ($m\ll n$), the regularized least-squares estimator provably recovers the correct support with high probability, both in the Gaussian case \citep{arag2015} and in the non-Gaussian case \citep{Loh13}. To enforce sparsity, we use the $\ell_1$ penalty on the coefficients of $W$ as a convex relaxation of $\ell_0$, which keeps the Bregman proximal mapping convex.

It is known that $\Theta = (\Expect(X^TX))^{-1}$, the precision matrix (inverse covariance matrix), can be written, 
\BEQ
\Theta = (I-W^*)\Omega^{-1}(I-W^*)^T,
\nonumber \EEQ
where $\Omega\in\reals^{n\times n}$ is the diagonal matrix with diagonal equal to the variances of $(\epsilon_j)_{j=1\ldots,n}$. This fact has been used by \citep{loh2014} to first estimate the moralized graph using the precision matrix and then use that information to restrict the search to DAGs matching the moralized graph uncovered.

\paragraph{Identifiability}
\cite{Shim2006} proved that if $(\epsilon_j)_{j=1\ldots,n}$ are jointly independent and non-Gaussian distributed with strictly positive density, the graph is identifiable from the joint distribution. \cite{Pete2014} extended this result to Gaussian errors with equal variances. 
We refer the reader to, \eg\cite{Pete2017} for a more complete discussion.

\paragraph{DAG Penalty}
We will first assume the $(\beta_{jk})$ coefficients are positive as our method is simpler in that case and generalize from there to negative edge weights. As introduced in NOTEARS \citep{Zhen2018}, we will use the smooth characterization of acyclicity recalled in the following proposition.

\begin{proposition}\label{prop:positive_dag}
A positive weighted adjacency matrix $W$ represents an acyclic graph if and only if $\Tr(I+\alpha W)^n~=~n$. 
\end{proposition}
\begin{proof}
Note that $W$ has a cycle of length $k\geq 1$ starting at $i$ if $[W^k]_{ii} > 0$. Since every cycle can be reduced to a cycle of length less or equal to $n$, 
\BEQ\BA{rl}\Tr(I + \alpha W)^n &= \sum_i \sum_{k=0}^n \binom{n}{k}\alpha^k [W^k]_{ii} \\[4pt]
&= n + \sum_i\sum_{k=1}^n \binom{n}{k}\alpha^k [W^k]_{ii}\\[4pt]
&\geq n, \nonumber
\EA\EEQ
with equality if and only if $W$ is acyclic. 
\end{proof}

Note that this choice of constraint is arbitrary, we could have equivalently chosen the constraint ${\Tr\,P(W)=n}$ for any polynomial $P$ with strictly positive coefficients and ${P(0)=1}$. As in \cite{Yu19}, we chose the form in Proposition~\ref{prop:positive_dag}, as its factored form allows for simpler expressions. 

Instead of a hard penalty, as in NOTEARS, we use a regularization term here, as in \citep{Ng2020}. Overall, we seek to solve
\BEQ\label{eq:formal_problem}
\BA{ll}
  \mbox{min.} & \frac{1}{m}\normF{X(I-W)}^2+\lambda \normone{W} + \mu\Tr (I + \alpha W)^n\\
  \mbox{s.t. }&\, W\geq 0,
\EA\EEQ
in the variable $W\in \reals^{n\times n}$, given samples $X\in \reals^{m\times n}$, $\lambda\geq 0$ and $\mu\geq 0$ control sparsity and DAG regularization respectively. The DAG regularization term in this problem has a high curvature and slows down convergence of classical gradient methods. In the next section, we look at how to efficiently solve this problem (locally) by measuring smoothness against a well-chosen, highly nonlinear kernel.

\paragraph{Handling Negative Coefficients}
Proposition~\ref{prop:positive_dag} holds only for positive edge weights. \cite{Zhen2018} alleviate this issue by squaring the adjacency matrix elementwise. However, squaring hurts regularity and we use a different approach here, writing $W$ the adjacency matrix as the difference of its positive and negative parts i.e. ${W = W^+ - W^-}$ where ${W^+ = \max(W, 0)}$ and ${W^- = \max(-W, 0)}$. 

Consider a new graph $\hat{G}$ obtained from $G$ by replacing every edge weight $\beta_{jk}$ by $|\beta_{jk}|$. Note that it's straightforward to see that $\Hat{G}$ is acyclic if and only if $G$ is acyclic. Moreover, the weighted adjacency matrix of $\Hat{G}$ is $W^++W^-$. 

The optimization problem on $G$ with acyclicity of $\Hat{G}$ is now very similar to the positive case, 
\BEQ\BA{rl}\label{eq:neg_weights_pb}
\minim &\mu \Tr \left(\idm + \alpha (W^+ + W^-)\right)^{n} \\[3pt]
    &+\frac{1}{m}\normF{X (\idm - W^++W^-)}^2 \\[3pt]
    &+ \lambda \normone{W^+} + \lambda \normone{W^-} \\[3pt]
\mbox{s. t.} & W^+, W^- \geq 0
,\EA\EEQ
where $W^+, W^- \in \reals^{n\times n}$.

We prove that ambiguous edges \ie where the weight is ill-defined cannot exist at critical points of the objective,
\begin{lemma}
No pair $W^+, W^-\in\reals_+^{n\times n} $ with indexes $(j,k)$ such that both $W^+_{jk}$ and $W^-_{jk}$ are non zeros are local minima of the objective of problem \eqref{eq:neg_weights_pb}.
\end{lemma}

\begin{proof}
Without loss of generality, assume ${W^+_{jk}\geq W^-_{jk}>0}$, let $\Tilde{W}^+$ be the same matrix as $W^+$ except at $(j, k)$ where $\Tilde{W}^+_{jk}=W^+_{jk}- W^-_{jk}$ and similarly with $\Tilde{W}^-$ where $\Tilde{W}^-_{ij}=0$, then $\Tilde{W}^+, \Tilde{W}^-\geq 0$, $\Tilde{W}^+-\Tilde{W}^- = W^+-W^-$ and $\normone{\Tilde{W}^\pm} < \normone{W^\pm}$. Moreover, $\Tilde{W}^\pm \leq W^\pm$ elementwise and the DAG regularization term is increasing in every entry of the matrix.

So ${(\Tilde{W}^+,\Tilde{W}^-)}$ has a strictly better objective than ${(W^+,W^-)}$.
\end{proof}

This means that the acyclicity constraints on $\Hat{G}$ and $G$ are completely equivalent.

\paragraph{Parameter estimation} We mostly focus on structure learning here, \ie~estimation of the support of the adjacency matrix and disregard the problem of parameter estimation \ie~learning the actual values of $(\beta_{jk})_{j,k=1\ldots,n}$. However, support estimation is the challenging task: once a correct graph $G$ is known, parameter estimation in the regularized least square setting is a convex problem.

\section{BREGMAN GRADIENT METHODS}\label{s:bregman}

Problems~\eqref{eq:formal_problem}~and~\eqref{eq:neg_weights_pb} are non-convex and the DAG penalty makes them highly nonlinear. Under more flexible relative smoothness assumptions introduced by \cite{Bolt18}, the performance and analysis of Bregman proximal gradient methods \citep{Birn11,Baus16,Lu18} becomes much closer to that of classical gradient methods. In practice, Bregman gradient methods allow us to neutralize at least part of the effect of nonlinearities by measuring smoothness against a more nonlinear kernel than that produced by norms. Here, we use a variant of these methods introduced by \cite{Drag19}, that uses dynamical step size for faster convergence. 

We first briefly recall the structure of Bregman gradient methods in the relative smoothness setting (aka the NoLips algorithm), as described in \citep{Bolt18}. Let $E$ be a Euclidean vector space endowed with an inner product $\la \cdot, \cdot \ra$. The method aims at solving non-convex composite minimization problems of the form
\BEQ \label{eq:min_problem}
\minim_{x\in \bar{C}} f(x)+g(x),
\EEQ
where $C$ is a nonempty convex open set in $\reals^n$, $f$ is a proper $C^1$ function with $\dom f \cap C \neq \emptyset$ and $g$ is a proper and lower semicontinuous function.

\subsection{Kernels, Bregman Divergences \& Relative Smoothness}
Let $h : E \rightarrow \reals$ be a differentiable strictly convex function, which is called the \textit{distance kernel}. It generates the \textit{Bregman distance}
\BEQ\label{eq:bregman_distance}
  D_h(x,y) = h(x) - h(y) - \la \nabla h(y), x-y \ra.
\EEQ
Note that $D_h$ is generally asymmetric, therefore it is not a proper distance and is sometimes referred to as a \textit{Bregman divergence}. However, we can deduce from the strict convexity of $h$ that $D_h$ enjoys a distance-like separation property: $D_h(x,x)~=~0$ and $D_h(x,y)~>~0$ for $x \neq y$.

We are now ready to define the notion of relative smoothness, also called L-smooth adaptability in \cite{Bolt18}.

\begin{definition}[Relative smoothness]
We say that a differentiable function $f : E \rightarrow \reals$ is L-smooth relatively to the distance kernel $h$ if there exists $L > 0$ such that 
\BEQ\label{eq:relative_smoothness}\tag{RelSmooth}
  f(x) \leq f(y) + \la \nabla f(y), x-y \ra + L D_h(x,y),
\EEQ
for all $x, y \in E$.
\end{definition}

There are several convenient ways of checking the relative smoothness assumption. In the differentiable case, an equivalent definition reads as follows. 

\begin{definition}[Relative smoothness II]
  We say that a $C^\infty$ function $f:\reals^{p\times p} \to \reals$ is L-smooth relatively to the distance kernel $h$, if for all $W\in \reals^{p\times p}$,
  \BEQ
  \nabla^2 f(W) \preceq L \nabla^2 h(W).
  \EEQ
\end{definition}

Note that if $h(x) = \frac{1}{2} \|x\|^2$ is the quadratic kernel, then $D_h(x,y) = \frac{1}{2}\|x-y\|^2$ and \eqref{eq:relative_smoothness} is equivalent to Lipschitz continuity of the gradient. By using different kernels, such as logarithms or power functions, it is possible to show that \eqref{eq:relative_smoothness} holds for functions that do not have a Lipschitz continuous gradient --- see \eg \citep{Baus16,Bolt18}.

\paragraph{Bregman Gradient Method} Now that we are equipped with a non-Euclidean geometry generated by $h$, we can define the Bregman proximal gradient map with step size $\gamma$ as follows
\BEA\label{eq:bregman_prox}
&& T_\gamma(x) \triangleq \\
&& \argmin_{u \in C} \left\{ g(u) + \la \nabla f(x), u-x \ra + \frac{1}{\gamma} D_h(u,x) \right\}.\nonumber
\EEA

The Bregman gradient method then simply iterates this mapping as in Algorithm~\ref{algo:nolips}.

\begin{algorithm}[H]
	\begin{algorithmic}
		\REQUIRE A function $h$ such that \eqref{eq:relative_smoothness} holds with relative Lipschitz constant $L$, and step size $0~<~\gamma~\leq~\frac{1}{L}$.
		\STATE Initialize $x_0 \in C$.
		\FOR{k = 1,2,\dots}
		\STATE 
		  $x^k \in T_\gamma(x^{k-1}) $
		\ENDFOR
	\end{algorithmic}
	\caption{Bregman Gradient Method}
	\label{algo:nolips}
\end{algorithm}

If $h$ is the squared Euclidean norm, we recover the projected gradient algorithm. In general, this of course assumes that the Bregman proximal map $T_\gamma(x)$ is simple to compute. We will see that in our case here, solving the iteration map is simply a convex problem. 

It can be easily be proved that, under the relative smoothness condition and with $\gamma \in (0, \frac{1}{L})$, the sequence $\{f(x^k)\}_{k\geq 0}$ is nonincreasing. Convergence towards a critical point of the problem \eqref{eq:min_problem} is established in \citep{Bolt18} under additional assumptions (boundedness of the sequence and Kurdyka-Lojasiewicz property) which will hold in our case.

\begin{figure*}[t]
  \centering
  \includegraphics[width=0.49\linewidth]{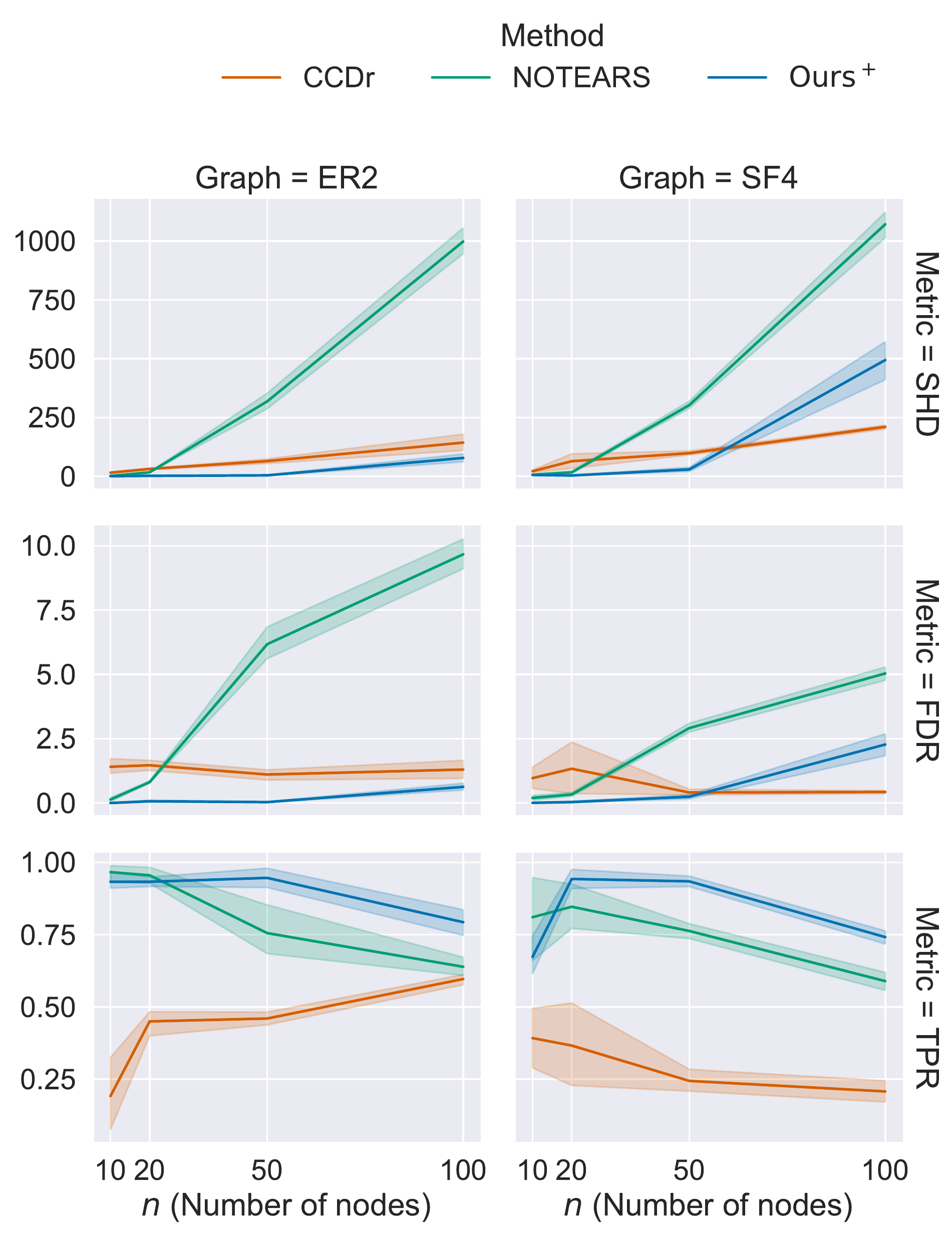}
  \includegraphics[width=0.49\linewidth]{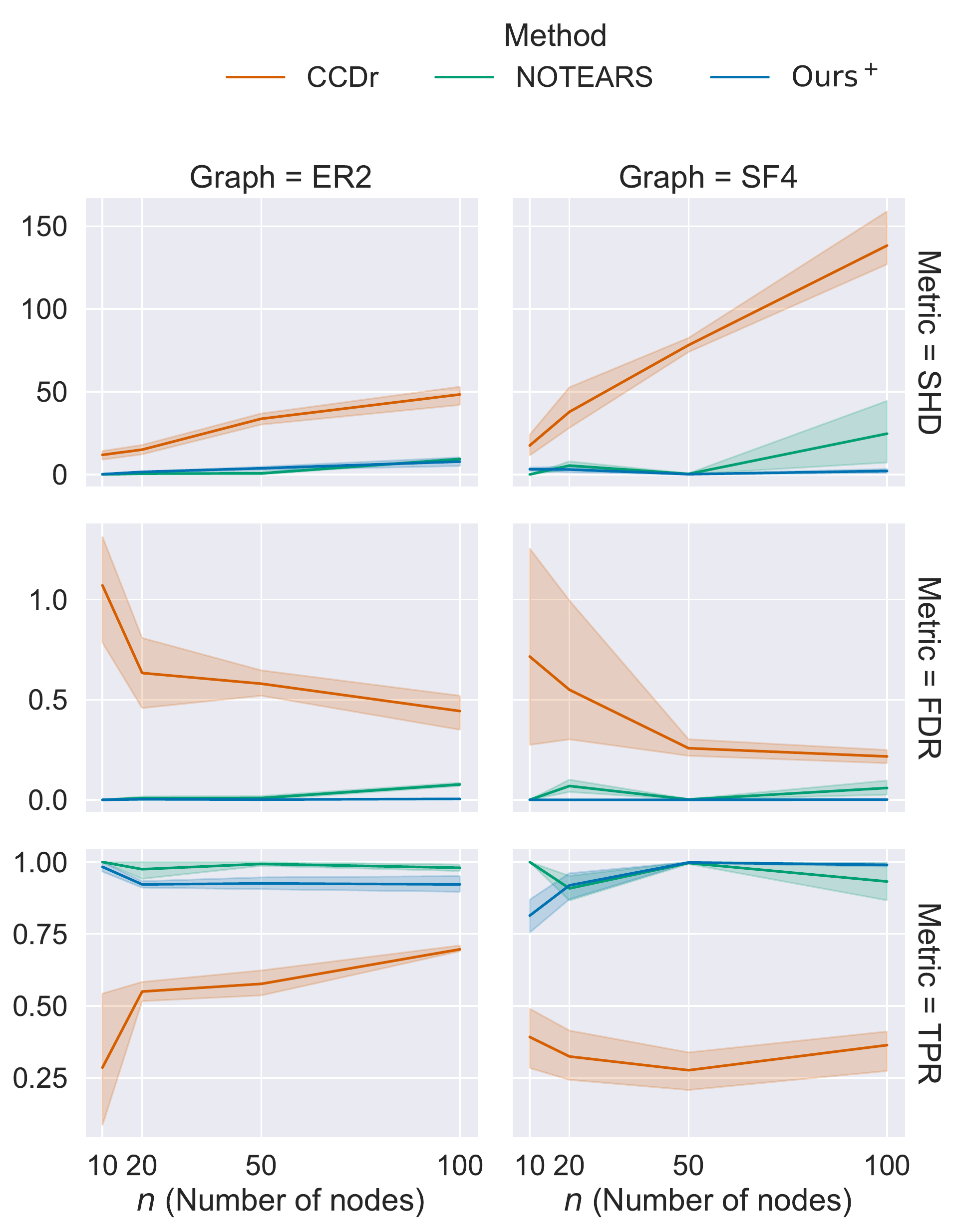}
  \vspace{.3in}
  \caption{Performance for $m=50$ (left) and $m=200$ (right) lower is better except for TPR. The method proposed here performs consitently well compared to the two other.}
  \label{fig:shd_fdr_vs_n}
\end{figure*}

\subsection{Bregman Method For Structure Learning}
In the case where $W^*$ is assumed positive, we can write problem~\eqref{eq:formal_problem} in the form~\eqref{eq:min_problem} with
\BEQ
\BA{rl}
f(W) &=\mu\Tr(\idm + \alpha W)^n\\
g(W) &= \frac{1}{m}\normF{X(I-W)}^2+\lambda \normone{W}
,\EA
\EEQ
two $C^\infty$ functions on $(0, +\infty)^{n\times n}$ that satisfy the Bregman Gradient methods' assumptions. 

\paragraph{Relative Smoothness} We choose a kernel which resembles our penalty function,
\BEQ\label{eq:kernel_pos}
h(W) = \mu(n-1)(1 + \alpha \normF{W})^n,
\EEQ
a $C^\infty$, convex function. We need to define the following convex subspace
\BEQ\BA{rl}
C_{\alpha} = \Big\{W\in\reals_+^{n\times n}\text{ such that } \sum_{ij}W_{ij}\geq \frac{n}{(n-2)\alpha}\Big\},\nonumber
\EA\EEQ
and get the following result.

\begin{theorem}
The DAG penalty $f$ is $1$-smooth relatively to $h$ on the convex space $C_{\alpha}$. 
\end{theorem}
The proof is left in the supplementary material. Note that we need to make the extra hypothesis that ${W^*\in C_{\alpha}}$, this assumption  is not very restrictive as we can choose $\alpha$, so this only impacts the regularity constants. 

In the general case, where $W^*$ takes on both positive and negative values, we change our DAG penalty function to:
\BEQ
\BA{rl}
f(W^+, W^-) &=\mu\Tr(\idm + \alpha (W^++W^-))^n
,\EA
\EEQ
a function with similar properties than in the positive case. We then define
\BEQ\label{eq:kernel_neg}
h(W^+, W^-) = \mu(n-1)(1 + \alpha \normF{W^++W^-})^n,
\EEQ
a $C^\infty$, convex function as our kernel, and the convex space
\BEQ\BA{rl}
C^+_{\alpha} &= \Big\{W^+,W^-\in\reals_+^{n\times n}\text{ s.t. }\\
&\qquad\sum_{ij}[W^++W^-]_{ij}\geq \frac{n}{(n-2)\alpha}\Big\}.\nonumber
\EA\EEQ
\begin{theorem}\label{th:rel_smooth_gen}
$f:\left(\reals_+^{n\times n}\right)^2\to \reals$ is 1-smooth relatively to $h$ on $C^+_{\alpha}$.
\end{theorem}
The proof, a natural extension of the proof in the positive case, is available to the reader in the supplementary material. The Bregman proximal gradient map then writes very similarly to \eqref{eq:bregman_prox_pos}.

\paragraph{Bregman Prox} Since $g$ is a convex function, computing the solution $T_\gamma(W_k)$ of the Bregman proximal gradient map means solving a convex minimization problem, written
\BEQ
\minim_{W \in C_\alpha} \Big\{ g(W) + \la \nabla f(W_k), W-W_k \ra + \frac{1}{\gamma} D_h(W,W_k) \Big\},\nonumber
\EEQ
which is again
\BEQ\label{eq:bregman_prox_pos}
\minim_{W \in C_\alpha} \Big\{\frac{1}{m}\normF{X(I-W)}^2 +\lambda \normone{W} \hfill + \la \nabla_k , W \ra + \frac{1}{\gamma} h(W) \Big\}.
\EEQ
where $\gamma>0$ and $\nabla_k = \nabla f(W_k) - \frac{1}{\gamma}\nabla h(W_k)$. In our case, this means minimizing a sum of convex functions with linear constraints, and can therefore be solved efficiently using off-the-shelf convex optimization solvers, \eg ECOS \citep{Doma13}, MOSEK \citep{mosek}. 

Having properly defined the functions and corresponding kernels as in  \eqref{eq:min_problem}, we can now apply algorithm \ref{algo:nolips}. As a final step to our method, we threshold the output matrix $W$ to get a binary adjacency matrix, to zero out negligible coefficients and because hard thresholding has been proven to reduce false discovery rate \citep{wang16notears}.

\begin{figure*}[ht]
  \centering
  \includegraphics[width=0.6\linewidth]{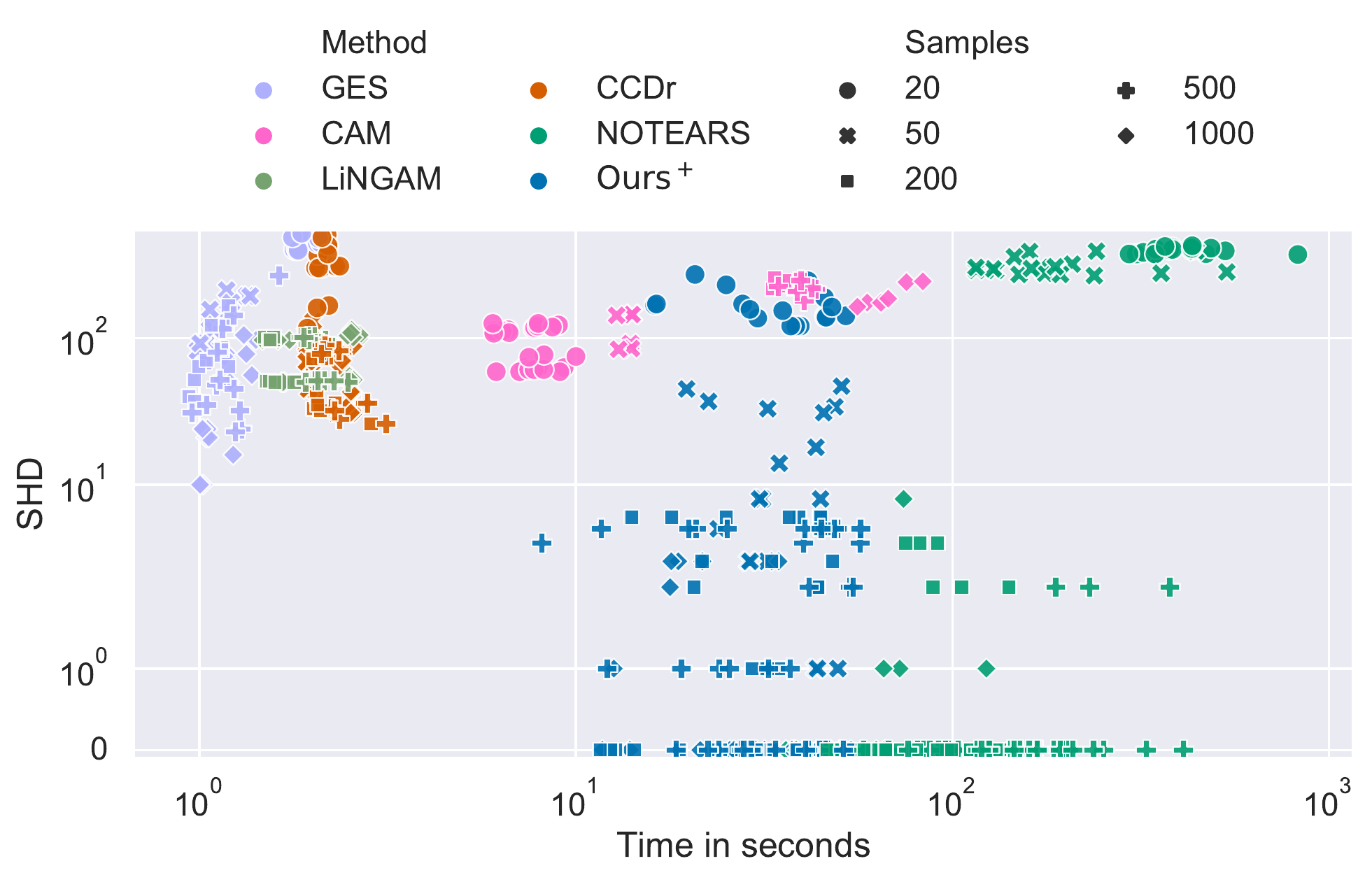}
  \vspace{.3in}
  \caption{Performance vs. runtime for $n=50$: for every dataset (with number of samples $m$, mean degree $k$, noise type and random seed), we plot Structural Hamming Distance (SHD) against run time of the corresponding algorithm.}
  \label{fig:shd_vs_time}
\end{figure*}

\section{EXPERIMENTS}

We compare our method against GES \citep{chic2002}, LiNGAM \citep{shim2014}, CAM \citep{Buhl2014}, CCDr \citep{arag2015ccdr} and NOTEARS \citep{Zhen2018}. Notably, it seems that LiNGAM fails to run on ill-conditioned covariance matrices. With the exception of NOTEARS, we use the implementation available as part of the Causal Discovery Toolbox by \cite{Kala2019}. However, we only report results against the following methods which perform better on the given tasks:
\BIT
\item \textbf{NOTEARS}\footnote{\href{https://github.com/xunzheng/notears}{https://github.com/xunzheng/notears}}\citep{Zhen2018}: we use the most recent code version that was updated to use a DAG constraint in form $\Tr(\idm+\alpha W\circ W)^n-n=0$ \citep{Yu19} instead of $\Tr e^{W\circ W}-n=0$ in the original paper.
\item \textbf{CCDr} \citep{arag2015}: they use a concave penalty, interpolation of the $\ell_0$ and $\ell_1$ penalty and reparametrize the Gaussian likelihood estimator into a convex objective. Finally they use coordinate descent on the obtained objective. 
\EIT
We did not test methods such as \cite{Yu19,ng2019masked,ng2019ae,lachapelle2020gradientbased,Ng2020} that run on GPU and necessitate substantial computational power.

\paragraph{Metrics} To compare the output of our model to the ground truth graph in synthetic examples, we let $\mathrm{TP}$ be the number of correctly detected edges and distinguish three error sources: $M$ counts the missing edges compared to the skeleton, $E$ counts the extra ones and $R$ the reversed edges from the ground truth directed graph. The most standard metric in Structure Learning is Structural Hamming Distance (SHD), the number of additions, deletions, reversals to go from our output graph to the true graph, hence $\SHD=M+E+R$. Other interesting metrics including False Discovery Rate (FDR) with $\FDR=(E+R)/p$ where $p$ is the true number of edges and True Positive Rate (TPR) where $\mathrm{TPR}=\mathrm{TP}/p$.

\paragraph{Negative Weights}
In our experiments, we notice that, even though the ground truth adjacency matrix $W^*$ contains both positive and negative coefficients (half in expectation), both NOTEARS and our algorithm, even when they perfectly recover the support, estimate all parameters to be positive. As an illustration, in our experiments, the proportion of correctly predicted edges by NOTEARS that were given a positive weight is ${99.4(\pm 1.5)\%}$ --- the minimum being $94.7\%$. Moreover, our algorithm restricted to positive coefficients \ie solving problem \eqref{eq:formal_problem} performs better than the general algorithm --- even when the graph has negative weighted edges. Therefore we present results with this algorithm denoted $\mathrm{Ours}^+$. We can't fully explain this phenomenon at this point.

\begin{figure}[h!]
  \centering
  \includegraphics[width=0.6\linewidth]{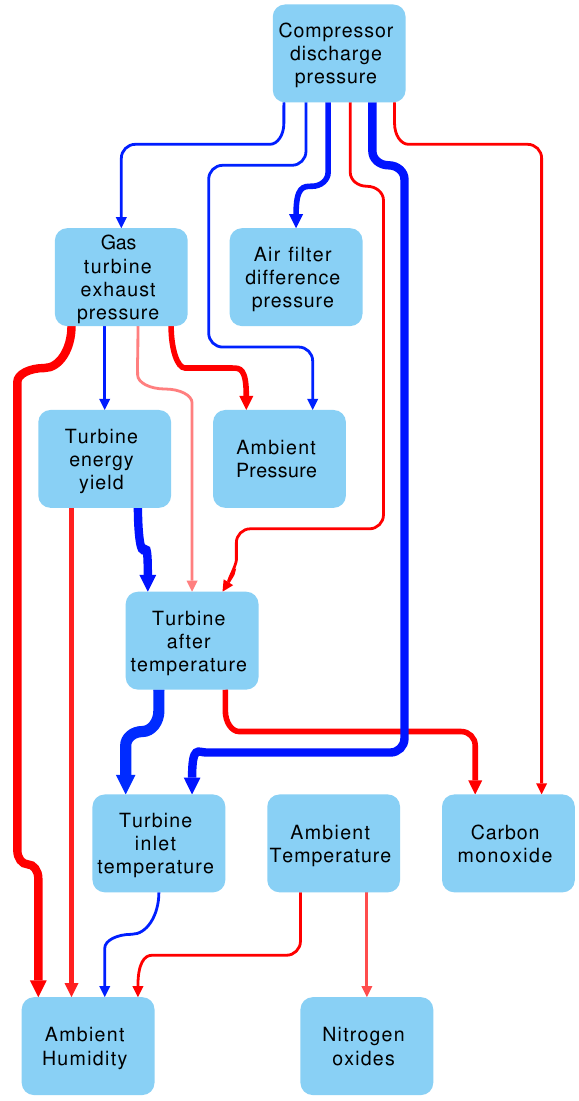} 
  \vspace{.3in}
  \caption{Graph learned on the Turbine dataset: color hue represents coefficient sign (blue is positive, red is negative); color intensity represents edge weight}
  \label{fig:turbine}
\end{figure}

\subsection{Synthetic Datasets}
We choose an experimental setup similar to the one of \cite{Zhen2018}. Our datasets vary on 5 different points: number of nodes, sparsity, graph type, noise type and number of samples. We generate $G$ from one of two types of random graphs: Erd\"os-Rényi (ER) or scale-free (SF). We sample graphs with $kn$ (${k=2,4}$) edges on average and denote the corresponding graph ER$k$/SF$k$. Given $G$, we assign random uniform weights to the edges in a fixed range. Then we generate $m \in [20, 50, 100, 200, 500, 1000]$ i.i.d. samples from the distribution entailed by the graph with one of three noise distributions: Gaussian, Exponential or Gumbel noise. 
\begin{figure*}[t!]
  \centering
  \includegraphics[width=0.8\linewidth]{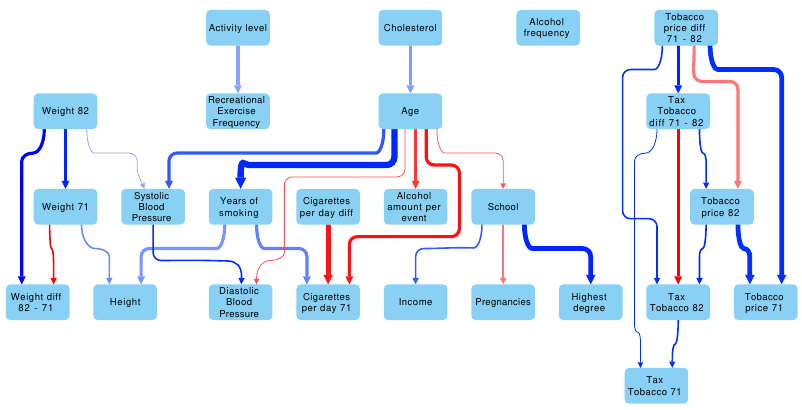}
  \caption{Graph learned on the NHEFS dataset: color hue represents coefficient sign (blue is positive, red is negative); width represents reliability from cross-validation; color intensity represents edge weight averaged over the $5$-fold cross validation}
  \label{fig:nhefs}
\end{figure*}

The results presented on Figure~\ref{fig:shd_fdr_vs_n} are aggregated over the three different noise types, several parameter choices ($\lambda$) and two random seeds. More detailed results (split by noise type and on ER$4$/SF$2$) can be found in Supplementary. 

From $m=200$ samples onward, both NOTEARS and our method learn the graph to quasi-perfection ($\SHD<10$) whereas CCDr stagnates to a sizable error level even with more data. On the other hand, when less data is available, CCDr becomes very competitive while NOTEARS has significant SHD but our method still behaves reasonably well, especially on sparser datasets such as ER$2$. 

Interestingly, we can see on Figure~\ref{fig:shd_vs_time} that the performance of our method increases steadily given more samples, unlike CCDr, while being approximately 10 times faster than NOTEARS, with better performance at a fixed number of samples. GES, LiNGAM and CAM are all much faster but with much worst SHD results, failing to get below 10 errors on datasets containing around $50$ or $100$ edges ($n=50$, $k=2, 4$).

\subsection{Real Data}

We run our algorithm on two real datasets. 
\paragraph{Gas Turbine} This dataset is composed of $36733$ instances of $11$ sensor measures from a Turkish gas turbine.\footnote{\href{https://archive.ics.uci.edu/ml/datasets/Gas+Turbine+CO+and+NOx+Emission+Data+Set}{Available here}} Results can be seen on Figure \ref{fig:turbine} and show key variables influencing emissions.

\paragraph{NHEFS}The NHANES I Epidemiologic Follow-up Study (NHEFS) is a clinical study that followed a cohort from 1971-75 yo 1982 \citep{hern2020}. Data include the initial examination, the 1982 follow up and data about their environment at both times, \eg tobacco prices. 

The dataset included missing data, which we inferred using the \texttt{IterativeImputer} of \texttt{scikit-learn}, which regresses the missing values using all the remaining variables. We selected numerical and ordinal variables only, and normalized them. Our preprocessed dataset is now composed of $m=1629$ patients with $n=25$ measurements each. We ran $5-$fold cross validation and report the graph where edges weights are averaged over splits (absolute weight shown as color intensity, sign as color) and frequency of the corresponding edge is shown as the edge width. The output of our algorithm is shown in Figure~\ref{fig:nhefs}. We observe some intuitive connections such as a positive weight \texttt{school} $\to$ \texttt{highest degree}. The algorithm, which is unsupervised, also managed to group together all variables related to tobacco prices.

\section{CONCLUSION AND FUTURE WORK}
Using penalties to enforce acyclicity and thanks to an appropriate choice of kernel, we develop a Bregman proximal gradient method for structure learning on linear structural causal models with good regularity properties, for which each Bregman proximal step amounts to solving a convex quadratic program. This allows the method to make longer steps and significantly improves convergence. The method has relatively low complexity and is uniformly competitive with existing algorithms on various synthetic data sets. We test it on two real data sets where it produces intuitive DAG structures.

\subsubsection*{Acknowledgements} 
A.A. is at the d\'epartement d’informatique de l'ENS, l'\'Ecole normale sup\'erieure, UMR CNRS 8548, PSL Research University, 75005 Paris, France, and INRIA. AA would like to acknowledge support from the {\em ML and Optimisation} joint research initiative with the {\em fonds AXA pour la recherche} and Kamet Ventures, a Google focused award, as well as funding by the French government under management of Agence Nationale de la Recherche as part of the "Investissements d'avenir" program, reference ANR-19-P3IA-0001 (PRAIRIE 3IA Institute).
\clearpage

\bibliography{biblio.bib,MainPerso.bib}

\begin{thebibliography}{}

\bibitem[ApS, 2019]{mosek}
ApS, M. (2019).
\newblock {\em MOSEK Fusion API for Python 9.2.26}.

\bibitem[Aragam et~al., 2015]{arag2015}
Aragam, B., Amini, A.~A., and Zhou, Q. (2015).
\newblock Learning directed acyclic graphs with penalized neighbourhood
  regression.

\bibitem[Aragam and Zhou, 2015]{arag2015ccdr}
Aragam, B. and Zhou, Q. (2015).
\newblock Concave penalized estimation of sparse gaussian bayesian networks.
\newblock {\em Journal of Machine Learning Research}, 16(1):2273--2328.

\bibitem[Bauschke et~al., 2016]{Baus16}
Bauschke, H.~H., Bolte, J., and Teboulle, M. (2016).
\newblock A descent lemma beyond lipschitz gradient continuity: first-order
  methods revisited and applications.
\newblock {\em Mathematics of Operations Research}, 42(2):330--348.

\bibitem[Birnbaum et~al., 2011]{Birn11}
Birnbaum, B., Devanur, N.~R., and Xiao, L. (2011).
\newblock Distributed algorithms via gradient descent for fisher markets.
\newblock In {\em Proceedings of the 12th ACM conference on Electronic
  commerce}, pages 127--136.

\bibitem[Bolte et~al., 2018]{Bolt18}
Bolte, J., Sabach, S., Teboulle, M., and Vaisbourd, Y. (2018).
\newblock First order methods beyond convexity and lipschitz gradient
  continuity with applications to quadratic inverse problems.
\newblock {\em SIAM Journal on Optimization}, 28(3):2131--2151.

\bibitem[B\"uhlmann et~al., 2014]{Buhl2014}
B\"uhlmann, P., Peters, J., and Ernest, J. (2014).
\newblock Cam: Causal additive models, high-dimensional order search and
  penalized regression.
\newblock {\em The Annals of Statistics}, 42(6):2526–2556.

\bibitem[Chickering, 2002]{chic2002}
Chickering, D.~M. (2002).
\newblock Optimal structure identification with greedy search.
\newblock {\em Journal of Machine Learning research}, 3(Nov):507--554.

\bibitem[Chickering et~al., 2004]{Chic2004}
Chickering, D.~M., Heckerman, D., and Meek, C. (2004).
\newblock Large-sample learning of bayesian networks is np-hard.
\newblock In {\em J. Mach. Learn. Res.}

\bibitem[Domahidi et~al., 2013]{Doma13}
Domahidi, A., Chu, E., and Boyd, S. (2013).
\newblock Ecos: An socp solver for embedded systems.
\newblock In {\em 2013 European Control Conference (ECC)}, pages 3071--3076.
  IEEE.

\bibitem[Dragomir et~al., 2019]{Drag19}
Dragomir, R.-A., d'Aspremont, A., and Bolte, J. (2019).
\newblock Quartic first-order methods for low rank minimization.
\newblock {\em arXiv preprint arXiv:1901.10791}.

\bibitem[Heckerman et~al., 1995]{heck1995}
Heckerman, D., Geiger, D., and Chickering, D.~M. (1995).
\newblock Learning bayesian networks: The combination of knowledge and
  statistical data.
\newblock {\em Machine learning}, 20(3):197--243.

\bibitem[Hernán and Robins, 2020]{hern2020}
Hernán, M.~A. and Robins, J.~M. (2020).
\newblock {\em Causal Inference: What If}.
\newblock Boca Raton: Chapman \& Hall, CRC.

\bibitem[Hu et~al., 2018]{Hu2018}
Hu, P., Jiao, R., Jin, L., and Xiong, M. (2018).
\newblock Application of causal inference to genomic analysis: Advances in
  methodology.
\newblock {\em Frontiers in Genetics}, 9:238.

\bibitem[Imbens, 2019]{Imbe2019}
Imbens, G. (2019).
\newblock {Potential Outcome and Directed Acyclic Graph Approaches to
  Causality: Relevance for Empirical Practice in Economics}.
\newblock NBER Working Papers 26104, National Bureau of Economic Research, Inc.

\bibitem[Kalainathan and Goudet, 2019]{Kala2019}
Kalainathan, D. and Goudet, O. (2019).
\newblock Causal discovery toolbox: Uncover causal relationships in python.

\bibitem[Kalisch and B{\"u}hlmann, 2007]{kali2007}
Kalisch, M. and B{\"u}hlmann, P. (2007).
\newblock Estimating high-dimensional directed acyclic graphs with the
  pc-algorithm.
\newblock {\em Journal of Machine Learning Research}, 8(Mar):613--636.

\bibitem[Koller and Friedman, 2009]{koll2009}
Koller, D. and Friedman, N. (2009).
\newblock {\em Probabilistic graphical models: principles and techniques}.
\newblock MIT press.

\bibitem[Lachapelle et~al., 2020]{lachapelle2020gradientbased}
Lachapelle, S., Brouillard, P., Deleu, T., and Lacoste-Julien, S. (2020).
\newblock Gradient-based neural dag learning.
\newblock In {\em International Conference on Learning Representations}.

\bibitem[Loh and Bühlmann, 2014]{loh2014}
Loh, P.-L. and Bühlmann, P. (2014).
\newblock High-dimensional learning of linear causal networks via inverse
  covariance estimation.
\newblock {\em Journal of Machine Learning Research}, 15.

\bibitem[Loh and Wainwright, 2013]{Loh13}
Loh, P.-L. and Wainwright, M.~J. (2013).
\newblock Regularized m-estimators with nonconvexity: Statistical and
  algorithmic theory for local optima.
\newblock In {\em Advances in Neural Information Processing Systems}, pages
  476--484.

\bibitem[Lu et~al., 2018]{Lu18}
Lu, Y., Fan, Y., Lv, J., and Noble, W.~S. (2018).
\newblock Deeppink: reproducible feature selection in deep neural networks.
\newblock In {\em Advances in Neural Information Processing Systems}, pages
  8676--8686.

\bibitem[Malinsky and Spirtes, 2018]{mali2018}
Malinsky, D. and Spirtes, P. (2018).
\newblock Causal structure learning from multivariate time series in settings
  with unmeasured confounding.
\newblock In {\em Proceedings of 2018 ACM SIGKDD Workshop on Causal Discovery},
  pages 23--47.

\bibitem[Ng et~al., 2019a]{ng2019masked}
Ng, I., Fang, Z., Zhu, S., Chen, Z., and Wang, J. (2019a).
\newblock Masked gradient-based causal structure learning.

\bibitem[Ng et~al., 2020]{Ng2020}
Ng, I., Ghassami, A., and Zhang, K. (2020).
\newblock On the role of sparsity and dag constraints for learning linear dags.
\newblock {\em ArXiv}, abs/2006.10201.

\bibitem[Ng et~al., 2019b]{ng2019ae}
Ng, I., yu~Zhu, S., Chen, Z., and Fang, Z. (2019b).
\newblock A graph autoencoder approach to causal structure learning.
\newblock {\em ArXiv}, abs/1911.07420.

\bibitem[Pearl, 2009]{Pear2009}
Pearl, J. (2009).
\newblock {\em Causality: Models, Reasoning and Inference}.
\newblock Cambridge University Press, USA, 2nd edition.

\bibitem[Peters et~al., 2017]{Pete2017}
Peters, J., Janzing, D., and Sch\"olkopf, B. (2017).
\newblock {\em Elements of Causal Inference: Foundations and Learning
  Algorithms}.
\newblock MIT Press, Cambridge, MA, USA.

\bibitem[Peters et~al., 2014]{Pete2014}
Peters, J., Mooij, J.~M., Janzing, D., and Sch{\"o}lkopf, B. (2014).
\newblock Causal discovery with continuous additive noise models.
\newblock {\em J. Mach. Learn. Res.}, 15:2009--2053.

\bibitem[Raskutti and Uhler, 2013]{Rask13}
Raskutti, G. and Uhler, C. (2013).
\newblock Learning directed acyclic graph models based on sparsest
  permutations.
\newblock {\em Stat}, 7(1):e183.

\bibitem[Sachs et~al., 2005]{Sach2005}
Sachs, K., Perez, O., Pe{\textquoteright}er, D., Lauffenburger, D.~A., and
  Nolan, G.~P. (2005).
\newblock Causal protein-signaling networks derived from multiparameter
  single-cell data.
\newblock {\em Science}, 308(5721):523--529.

\bibitem[Shimizu, 2014]{shim2014}
Shimizu, S. (2014).
\newblock Lingam: Non-gaussian methods for estimating causal structures.
\newblock {\em Behaviormetrika}, 41(1):65--98.

\bibitem[Shimizu et~al., 2006]{Shim2006}
Shimizu, S., Hoyer, P.~O., Hyv\"{a}rinen, A., and Kerminen, A. (2006).
\newblock A linear non-gaussian acyclic model for causal discovery.
\newblock {\em J. Mach. Learn. Res.}, 7:2003–2030.

\bibitem[Uhler et~al., 2013]{Uhle2013}
Uhler, C., Raskutti, G., B{\"u}hlmann, P., and Yu, B. (2013).
\newblock Geometry of the faithfulness assumption in causal inference.
\newblock {\em The Annals of Statistics}, 41(2):436–463.

\bibitem[Van~de Geer and B\"uhlmann, 2013]{Van-13}
Van~de Geer, S. and B\"uhlmann, P. (2013).
\newblock $\ell_0$ penalized maximum likelihood for sparse directed acyclic
  graphs.
\newblock {\em The Annals of Statistics}, 41(2):536--567.

\bibitem[Wang et~al., 2016]{wang16notears}
Wang, X., Dunson, D., and Leng, C. (2016).
\newblock No penalty no tears: Least squares in high-dimensional linear models.
\newblock In Balcan, M.~F. and Weinberger, K.~Q., editors, {\em Proceedings of
  Machine Learning Research}, volume~48, pages 1814--1822, New York, New York,
  USA. PMLR.

\bibitem[Yu et~al., 2019]{Yu19}
Yu, Y., Chen, J., Gao, T., and Yu, M. (2019).
\newblock Dag-gnn: Dag structure learning with graph neural networks.
\newblock In {\em Proceedings of the 36th International Conference on Machine
  Learning}.

\bibitem[Zheng et~al., 2018]{Zhen2018}
Zheng, X., Aragam, B., Ravikumar, P., and Xing, E.~P. (2018).
\newblock Dags with no tears: Continuous optimization for structure learning.
\newblock In {\em NeurIPS}.

\end{thebibliography}


\end{document}


%

%

\onecolumn
\aistatstitle{A Bregman Method for Structure Learning on\\ Sparse Directed Acyclic Graphs}
\section{Dynamic NoLips}

In this work, we use the dynamic variant of NoLips of \cite{Drag19} that uses more aggressive step size for faster convergence, 

\begin{algorithm}[H]
	\begin{algorithmic}
		\REQUIRE A function $h$ such that $f$ is smooth relatively to $h$,  initial step size $\gamma_0>0$ and maximal step size $\gamma_{\max}$.
		\STATE Initialize $W \in C$ such that $\Psi(W)<\infty$.
		\STATE $\gamma \gets \gamma_{0}$
		\REPEAT
    		\REPEAT 
    		\STATE $W^+ \gets T_\gamma(W) $
    		\IF{the decrease condition \eqref{eq:suff_decrease} is not verified}
    	        \STATE $\gamma \gets \gamma/2$
    		\ENDIF
    		\UNTIL{decrease condition \eqref{eq:suff_decrease} is verified}
    		\STATE $W \gets T_\gamma(W) $
    		\STATE $\gamma \gets \min(2\gamma, \gamma_{\max}) $
		\UNTIL{convergence criterion}
	\end{algorithmic}
	\caption{Dynamic NoLips}
	\label{algo:dynnolips}
\end{algorithm}

The sufficient decrease condition is, for $W^+ = T_{\gamma}(W)$, 
\BEQ\BA{rl}\label{eq:suff_decrease}
f(W^+) \leq f(W) + \la\nabla f(W), W^+ - W\ra + \frac{1}{\gamma}D_h(W^+, W).
\EA\EEQ

\cite{Drag19} proves that this modified version of NoLips allows for larger step sizes while maintaining convergence guaranties. We refer the reader to their work for further information and proofs.  

As a convergence criterion, we use convergence of the $L_2$ error, \ie  for $\ell_k = \frac{1}{m}\normF{X(\idm - W_k)}^2$ being the $L_2$ error at step $k$, we stop when $\left|\frac{\ell_k - \ell_{k-1}}{\ell_{k-1}}\right|\leq \tau$, with an user-defined $\tau$.

\section{Proofs of relative smoothness}

This section provides proofs for theorems 3.3 and 3.4 that were left out in the main paper. 

\subsection{Positive case - Proof of Theorem 3.3}

Recall that in this case, for $W\in\reals^{n\times n}_+$,
\BEQ\BA{rl}\label{eq:def_DAG_pos}
f(W) &= \mu\Tr(\idm+\alpha W)^n\\
h(W) &= \mu(n-1)(1+\alpha \normF{W})^n
\EA\EEQ
We will prove, 
\begin{theorem}[3.3 in main paper]
In the set $ \mathcal{M}_{\alpha} = \left\{W\in\reals_+^{n\times n}\text{ such that } \normF{W}\geq \frac{1}{(n-2)\alpha}\right\}$, $f$ is $1$-smooth relatively to $h$ \ie, 
\BEQ\BA{rl}
\left|\nabla^2f(W)[H, H]\right|\leq \nabla^2h(W)[H, H].\nonumber
\EA\EEQ
\end{theorem}
\clearpage
\begin{proof}
Note that both $f$ and $h$ are $C^\infty$ on $\mathcal{M}_{\alpha}$. Note that $f(W) = \mu\sum_{k=0}^n\binom{n}{k}\alpha^k\Tr(W^k)$, using this form and Taylor series, it is easier to get to:
\BEQ\BA{rl}\label{eq:hess_DAG}
\frac{1}{2}\nabla^2 f(W)[H,H] = \mu\sum_{k=2}^n\binom{n}{k}\alpha^k\Tr(X_kH),
\EA\EEQ
where $X_k = \sum_{j=0}^{k-2}(j+1)W^jHW^{k-2-j}$.

So:
\BEQ\label{eq:hess_DAG_ineq}
\BA{rl}
|\nabla^2 f(W)[H, H]| \leq \mu n(n-1)\alpha^2(1+\alpha ||W||)^{n-2}||H||^2 
\EA\EEQ
Moreover for $W\in \mathcal{M}_\alpha$, and $\Tilde{W} = \frac{W}{\normF{W}}$, we have:
\BEQ
\BA{ll}\label{eq:hess_kernel_pos}
\nabla^2 h(W)[H, H]&= \mu n(n-1)(1+\alpha \|W\|)^{n-2}\left[
 (n-1)\alpha^2(\Tr\Tilde{W}^TH)^2 
 + \alpha\frac{1 + \alpha\|W\|}{||W||}\left(||H||^2-(\Tr\Tilde{W}^TH)^2\right)\right]\\
 &= \mu n(n-1)(1+\alpha \|W\|)^{n-2}\Bigg[\underbrace{\left(
 (n-1)\alpha^2-\alpha\frac{1 + \alpha\|W\|}{||W||}\right)}_{\geq 0}(\Tr\Tilde{W}^TH)^2
 + \alpha\underbrace{\frac{1 + \alpha\|W\|}{||W||}}_{\geq \alpha}||H||^2\Bigg]
 \\
 &\geq \mu n(n-1)\alpha^2(1+\alpha \|W\|)^{n-2}||H||^2
\EA\EEQ

So we proved that $|\nabla^2 f(W)[H, H]|\leq \nabla^2 h(W)[H, H]$.
\end{proof}

Now, note that using the inequality $\|x\|_1 \leq \sqrt{n}\|x\|_2$ for $x\in \reals^n$, the set
\BEQ\BA{rl}
C_{\alpha} = \Big\{W\in\reals_+^{n\times n}\text{ such that } \normone{W}\geq \frac{n}{(n-2)\alpha}\Big\},\nonumber
\EA\EEQ
is a convex subset of $\mathcal{M}_{\alpha}$. 

\subsection{General case - Proof of Theorem 3.4}
\newcommand{\funcsum}{\mathrm{s}}

Let $f:\reals^{n\times n}\to\reals$ and $h:\reals^{n\times n}\to\reals$ be the functions  of the positive case (see Eq. \eqref{eq:def_DAG_pos}) and the function sum $\funcsum:A,B\to A+B$. 

In this case:
\BEQ\BA{rl}
\hat{f}(W^+, W^-) &= \mu\Tr\left(\idm+\alpha (W^++W^-)\right)^n = f\circ\funcsum(W^+, W^-)\\
\hat{h}(W^+, W^-) &= \mu(n - 1)\left(1+\alpha \normF{W^++W^-}\right)^n = h\circ\funcsum(W^+, W^-)
\EA\EEQ
We will prove the following lemma that allows to adapt the proof of the positive case to the general case. 

\begin{lemma}
For all $C^\infty$ functions $\hat{f}:\left(\reals_+^{n\times n}\right)^2\to \reals$ and $f:\reals^{n\times n}_+\to \reals$ such that $\hat{f} = f\circ \funcsum$, for all $W\in \left(\reals_+^{n\times n}\right)^2$ and $H \in \left(\reals^{n\times n}\right)^2$, we have:
\BEQ\BA{rl}
\la\nabla \hat{f}(W), H\ra &= \la \nabla f(\funcsum(W)), \funcsum(H)\ra\\
\nabla^2 \hat{f}(W)[H, H] &= \nabla^2 f(\funcsum(W))[\funcsum(H), \funcsum(H)]
\EA\EEQ
\end{lemma}
\begin{proof}
Composition with a linear operator.
\end{proof}

So in our case, for all $W \in \left(\reals_+^{n\times n}\right)^2$ such that $\normF{\funcsum(W)} \geq \frac{1}{(n-2)\alpha}$, and all $H \in \left(\reals^{n\times n}\right)^2$, we have:

\BEQ\BA{rl}
|\nabla^2 \hat{f}(W)[H, H]| = |\nabla^2 f(\funcsum(W))[\funcsum(H), \funcsum(H)]|
\leq \nabla^2 h(\funcsum(W))[\funcsum(H), \funcsum(H)]
 = \nabla^2 \hat{h}(W)[H, H]
\EA\EEQ
where $\hat{h}=h\circ\funcsum$. 
So $\hat{f}$ is $1$-smooth relatively to $\hat{h}$ on $\mathcal{M}^+_{\alpha}=\left\{W\in\left(\reals_+^{n\times n}\right)^2\,|\, s(W)\in\mathcal{M}_{\alpha}\right\}$, so on the convex subset $
C^+_{\alpha} = \Big\{W\in\left(\reals_+^{n\times n}\right)^2\,|\,s(W)\in C_\alpha\Big\}$.

\section{Experimental details \& further results}

\subsection{Synthetic datasets}
We found $\alpha = 0.1/n$ and $\mu=100$ to be working well and used it in all our synthetic data experiments. We used $\lambda\in [0, 10^{-6}, 10^{-4}]$ and averaged results for both NOTEARS and our algorithm, however the results are not very sensitive to the choice of $\ell_1$ penalty. We used $\tau=10^{-7}$ across experiments. 

Figures \ref{fig:shd_m50} and \ref{fig:shd_m200} show that both NOTEARS and our algorithm seem to perform similarly with varying noise type and graph architecture, however, CCDr seems to perform significantly better when graphs are sparser ($k=2$). 

A short Python demonstration is available as part of the supplementary material. 
 
\begin{figure}
    \centering
    \includegraphics[width=\linewidth]{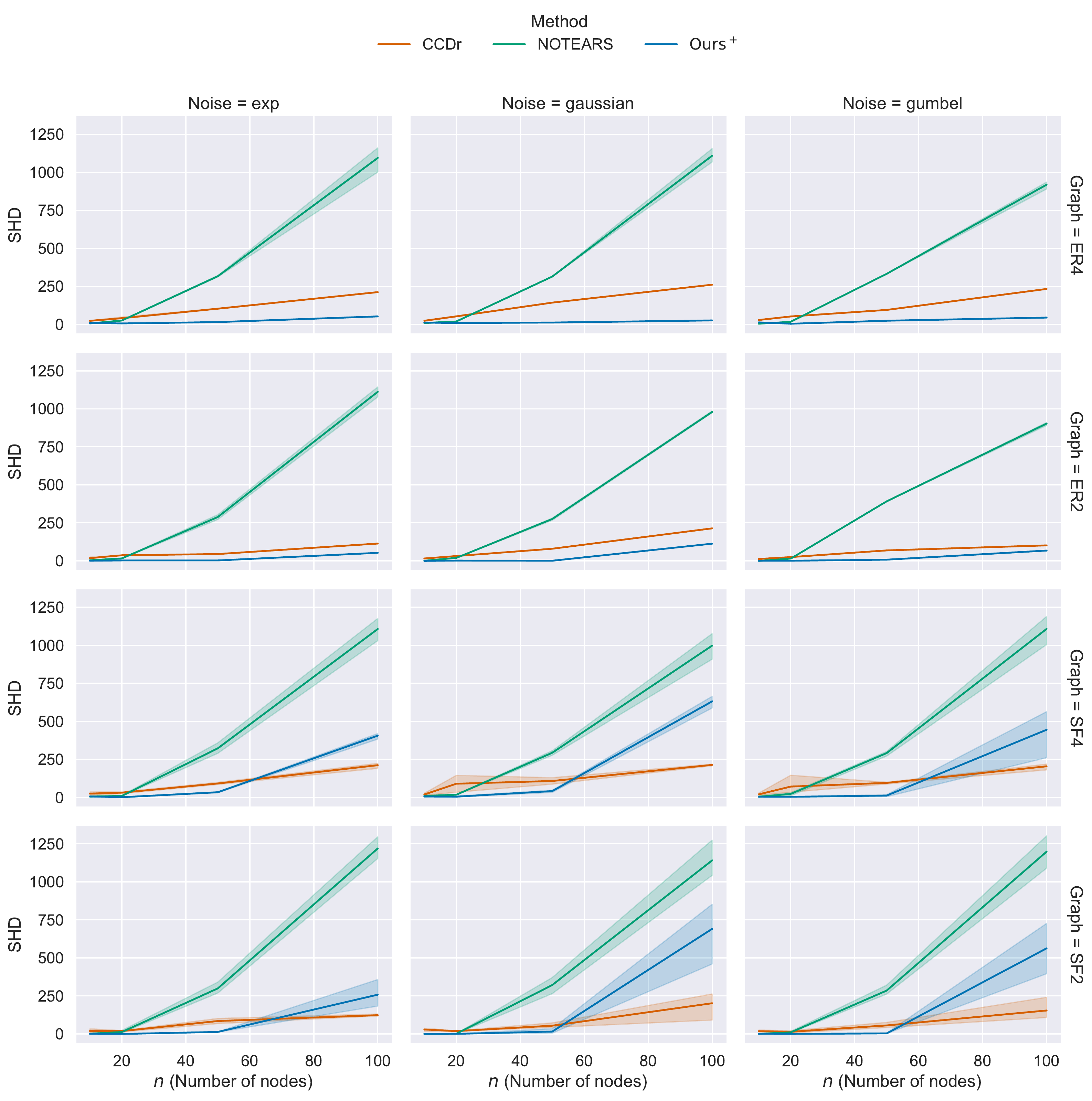}
    \caption{Structural Hamming Distance (SHD) on 12 different types of datasets (graph type, mean degree and noise type differ) for $m=50$}
    \label{fig:shd_m50}
\end{figure}

\begin{figure}
    \centering
    \includegraphics[width=\linewidth]{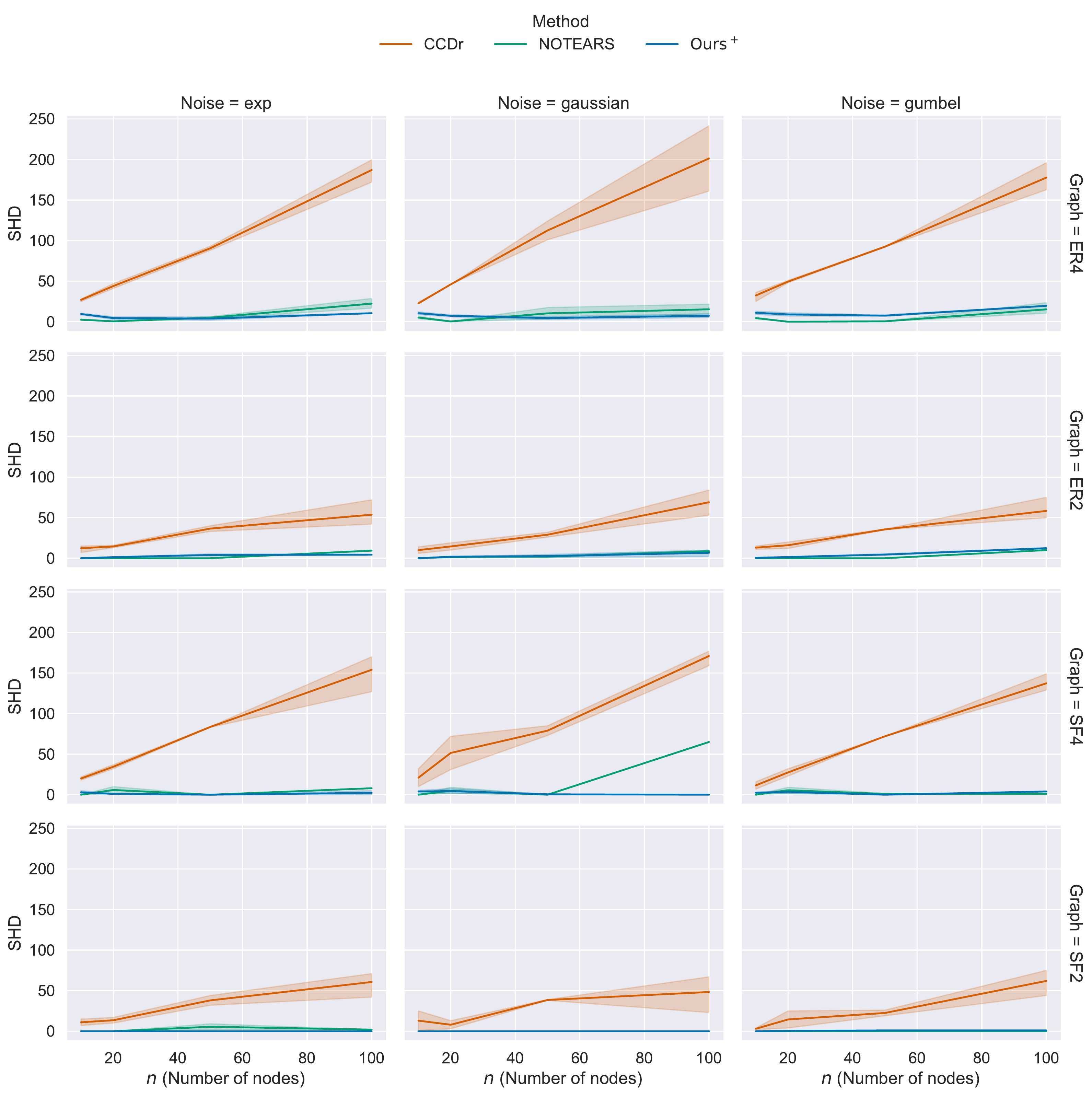}
    \caption{Structural Hamming Distance (SHD) on 12 different types of datasets (graph type, mean degree and noise type differ) for $m=200$}
    \label{fig:shd_m200}
\end{figure}


\bibliography{biblio.bib,MainPerso.bib}